\documentclass[wcp]{jmlr}

\usepackage{longtable}

\usepackage{booktabs}

\usepackage{amsmath}
\makeatletter
\newcommand{\distas}[1]{\mathbin{\overset{#1}{\kern\z@\sim}}}%
\newsavebox{\mybox}\newsavebox{\mysim}
\newcommand{\distras}[1]{%
	\savebox{\mybox}{\hbox{\kern3pt$\scriptstyle#1$\kern3pt}}%
	\savebox{\mysim}{\hbox{$\sim$}}%
	\mathbin{\overset{#1}{\kern\z@\resizebox{\wd\mybox}{\ht\mysim}{$\sim$}}}%
}
\makeatother
\usepackage{amsfonts}
\usepackage{amssymb}
\usepackage{graphicx}
\usepackage{multirow}

\def \xx {\mathbf{x}}
\def \ww {\mathbf{w}}
\def \X  {\mathcal{X}}
\def \Y  {\mathcal{Y}}
\def \R  {\mathcal{R}}

\def \D {\mathcal{D}}
\usepackage{epstopdf}
\title[Robust Decision Trees]{On the Robustness of Decision Tree Learning under Label Noise}

  \author{\Name{Aritra Ghosh} \Email{aritraghosh.iem@gmail.com}\\
 \addr Microsoft India (R \& D) Pvt. Ltd., Bangalore
\AND
\Name{Naresh Manwani} \Email{nareshmanwani@gmail.com}\\
\addr Microsoft India (R \& D) Pvt. Ltd., Bangalore
\AND
\Name{P. S. Sastry} \Email{sastry@ee.iisc.ernet.in}\\
\addr Electrical Engineering, Indian Institute of Science, Bangalore}

\begin{document}
	
	\maketitle
	
	\begin{abstract}
		In most practical problems of classifier learning, the training data suffers from the label noise. Hence, it is important to understand how robust is a learning algorithm to such label noise. 
This paper presents some theoretical analysis to show that many popular decision tree algorithms are robust to symmetric label noise under large sample size. We also present some sample complexity results which provide some bounds on the sample size for the robustness to hold with a high probability. Through extensive simulations we illustrate this robustness. 
	\end{abstract}
	
	\section{Introduction}
	Decision tree is among the most widely used machine learning approaches \citep{Wu2007}.
Interpretability, applicability to all types of features, less demands on data pre-processing and scalability 
are some of the reasons for its popularity. In general, decision tree is learnt in a top down greedy fashion where, at each node, a split rule is learnt by minimizing some objective function. 
	
	For learning a decision tree classifier, we make use of labeled training data.
	When the class labels in the training data may be incorrect, it is referred to as label noise. Subjectivity and other errors in human labeling, measurement errors, insufficient feature space are some of the main reasons behind label noise. In many large data problems, labeled samples are often obtained through crowd sourcing and the unreliability of such labels is another reason for label noise. 
Learning from positive and unlabeled samples can also be cast as a problem of learning under label noise \citep{du2014analysis}. Thus, learning classifiers
	in the presence of label noise is an important problem \citep{frenay2014classification}. 
	It is generally accepted that among all the classification methods, decision tree is probably closest to {\lq off-the-shelf\rq} method which has all the desirable properties including robustness to outliers \citep{hastie2005elements}.
	
While there are many results about generalization bounds for decision trees \citep{mansour2000generalization,kearns1998fast}, not many theoretical results are known about the robustness of decision tree learning in presence of label noise. 
It is observed that label  noise in the training data increases size of the learnt tree; detecting and removing noisy examples improves the learnt tree \citep{brodley1999identifying}.
Recently, \cite{nettleton2010study} empirically studied robustness of different classifiers under label noise. While decision tree learning is better than SVM or logistic regression in terms of robustness to label noise, it is also seen that naive Bayes is more robust than decision trees.
In this paper, we present a theoretical study of such robustness properties of decision trees. 

	Recently, many analytical results are reported on robust learning of classifiers, using the framework of risk minimization.
The robustness or noise tolerance of risk minimization depends on the loss function used. \cite{long2010random} proved that any convex potential loss is not robust to uniform or symmetric label noise. Another result is that some of the standard convex losses are not robust to symmetric label noise while the 0-1 loss is  \citep{manwani2013noise}. 
It is noted by \cite{du2014analysis} that convex surrogates losses are not good for learning from positive and unlabeled data. A general sufficient condition on the loss function for risk minimization to be robust is derived in  \citep{ghosh2015making}. The 0-1 loss, sigmoid loss and ramp loss are shown to satisfy this condition while convex losses such as hinge loss (used in SVM) and the logistic loss do not satisfy this condition. 
	Interestingly, it is possible to have a convex loss (which is not a convex potential) that satisfies this sufficient condition and the corresponding risk minimization essentially amounts to a highly regularized SVM  \citep{van2015learning}. Robust risk minimization strategies under the so called class-conditional (or asymmetric) label noise are also proposed  \citep{natarajan2013learning,DBLP:conf/colt/ScottBH13}. Some sufficient conditions for robustness of risk minimization under 0-1 loss, ramp loss and sigmoid loss when the training data is corrupted with most general non-uniform label noise are also presented in \citep{ghosh2015making}. None of these results are applicable for decision trees because the popular decision tree learning algorithms cannot be cast as 
risk minimization.

	In this paper, we analyze learning of decision trees under label noise. We consider some of the popular impurity function based methods for learning of decision trees. 
We show, in the large sample limit, that under symmetric or uniform label noise the split rule that optimizes the  objective function under noisy data is the same as that under noise-free data. 
We explain how this results in the learning algorithm being robust to label noise, under the assumption that the number of samples at every node is large. We also derive some sample complexity bounds to indicate how large a sample we need at a node. We also explain how these results indicate robustness of random forest also. 
We present empirical results to show that trees learnt with noisy data give accuracies that are comparable with those learnt with noise-free data. 
	We also show empirically that the random forests algorithm is robust to label noise. 
For comparison we also present results obtained with SVM algorithm. 
	
	\section{Label Noise and Decision Tree Robustness} 
	\label{sec:problem}
	In this paper, we only consider binary decision trees for binary classification. 
	We use the same notion of noise tolerance as in \citep{manwani2013noise,van2015learning}. 
	\subsection{Label Noise}
	Let $\X \subset \R^d$ be the feature space 
and let $\Y=\{1,-1\}$ be the class labels. 
	Let $S=\{(\xx_1,y_{\xx_1}),(\xx_2,y_{\xx_2}),\ldots, (\xx_N,y_{\xx_N})\} \in (\X \times \Y)^N$ be the 
	{\em ideal} noise-free data drawn {\em iid} from a fixed but unknown distribution $\mathcal{D}$ over $\X \times \Y$. The learning algorithm does not have access to this data. 
	The noisy training data given to the algorithm is $S^{\eta} = \{(\xx_i,\tilde{y}_{\xx_i}),i=1,\cdots,N\}$, 
	where $\tilde{y}_{\xx_i}=y_{\xx_i}$ with probability $(1-\eta_{\xx_i})$ and $\tilde{y}_{\xx_i}=-y_{\xx_i}$ with 
	probability $\eta_{\xx_i}$. As a notation, for any $\xx$, $y_{\xx}$ denotes its `true' label while $\tilde{y}_{\xx}$ denotes the noisy label. Thus, $\eta_{\xx} = \Pr[y_{\xx} \neq \tilde{y}_{\xx} \; | \; \xx]$.  We use $\mathcal{D}^{\eta}$ to denote the joint probability distribution of $\xx$ and $\tilde{y}_{\xx}$. 
	
	We say that the noise is {\em uniform} or {\em symmetric} if $\eta_{\xx}=\eta, \; \forall \xx$. Note that, under symmetric noise, a sample having wrong label is independent of the feature vector and the `true' class of the sample. 
	Noise is said to be {\em class conditional} or {\em asymmetric} if $\eta_{\xx}=\eta_+, \;$ 
for all patterns of class $+1$ 
	and $\eta_{\xx}=\eta_-,\;$
for all patterns of class $-1$. 
When noise rate $\eta_{\xx}$ is a general function of $\xx$, it is termed as {\em non-uniform} noise. Note that the value of $\eta$ is unknown to the learning algorithm. 
	
	\subsection{Criteria for Learning Split Rule at a Node of Decision Trees}
	\label{sec:crit_split}
	Most decision tree learning algorithms grow the tree in top down fashion starting with all training data at the root node. At any node, the algorithm selects a split rule to optimize a criterion and uses that split rule to split the data into the left and right children of this node; then the same process is recursively applied to the children nodes till the node satisfies the criterion to become a leaf. Let ${\cal F}$ denote a set of split rules. Suppose, a split rule $f \in {\cal F}$ at a node $v$, sends a fraction $a$ of the samples at $v$ to the left child $v_l$ and the remaining fraction $(1-a)$ to the right child $v_r$. Then many algorithms select a $f \in {\cal F}$ to maximize a criterion 
\begin{equation}
	 C(f) = G(v) - (a G(v_l) + (1-a) G(v_r))
\label{eq:criterion} 
\end{equation}
	where $G(\cdot)$ is a so called impurity measure. 
	There are many such impurity measures. Of the samples at any node $v$, suppose a fraction $p$ are of positive class and a fraction $q = (1-p)$ are of negative class. Then the gini impurity is defined by $G_{\mbox{Gini}}=2pq$ \citep{cart84}; entropy based impurity is defined as $G_{\mbox{Entropy}}=-p\log p-q\log q$ \citep{quinlan1986induction}; and misclassification impurity is defined as $G_{\mbox{MC}}=\min\{p,q\}$. 
	Often the criterion $C$ is called the {\em gain}. Hence, we also use $\mbox{gain}_{\mbox{Gini}}(f)$ to refer to $C(f)$ when $G$ is $G_{\mbox{Gini}}$ and similarly for other impurity measures. 
	
	A split criterion different from impurity is twoing rule, first proposed by \cite{cart84}. Consider a split rule $f$ at a node $v$. Let $p_l$ $(p_r)$, $q_l$ $(q_r)$ be the fraction of positive and negative class samples at the left (right) child $v_l$ $(v_r)$. (We have, $a p_l + (1-a)p_r=p,\quad a q_l + (1-a)q_r =q$, $p$ and $q$ are the fractions for parent node $v$). Then twoing rule selects $f\in {\cal F}$ which maximizes
	$G_{\mbox{Twoing}}(f)=a(1-a)[|p_l-p_r|+|q_l-q_r|]^2/4$.

	\subsection{Noise Tolerance of Decision Tree}
By noise tolerance we desire the following. A decision tree learnt with noisy labels in training data should have the same test error (on noise-free test set) as that of the tree learnt using noise-free training data. 
	One way of achieving such robustness is if the decision tree learning algorithm learns the same tree in presence of label noise as it would learn with noise free data.\footnote{For simplicity, we do not consider pruning of the tree.} Since label noise is random, on any specific noise-corrupted training data, the tree learnt would  also be random. Hence, we say the learning method is robust if, in the limit as training set size goes to infinity, the algorithm learns the same tree with noisy as well as noise-free training data.   We then argue that this implies we learn the same tree  (with a high probability) if given sufficient number of samples. We also provide sample complexity results for this.
Below, we formalize this notion. 
	\begin{definition}
		\label{RobustSplit}
		A split criterion $C$ is said to be {\em noise-tolerant} if
		\[\arg\min_{f\in \mathcal{F}} C(f)=\arg\min_{f\in \mathcal{F}} C^{\eta}(f)\]
	\end{definition}
	where $C(f)$ is the value of the split criterion $C$ for a split rule $f \in \mathcal{F}$ on noise free data  and $C^{\eta}(f)$ is the value of the criterion function for $f$ on noisy data, in the limit as the data size goes to infinity. 

Let the decision tree learnt from training sample $S$ be represented as $LearnTree(S)$ and let the classification of any $\xx$ by this tree be represented as $LearnTree(S)(\xx)$. 
	\begin{definition}
		A decision tree learning algorithm $LearnTree$ is said to be noise-tolerant if the probability of misclassification, under the noise-free distribution, of the tree learnt with noisy samples is same as that learnt with noise-free samples. That is, 
		\[P_{\D}(LearnTree(S)(\xx)\neq y_{\xx}) =P_{\D}(LearnTree(S^{\eta})(\xx)\neq y_{\xx})\]
	\end{definition}
	Note that for the above to hold it is sufficient if $LearnTree(S)$ is same as $LearnTree(S^{\eta})$. 	

	\section{Theoretical Results}
Robustness of decision tree learning requires the robustness of the split criterion at each non-leaf node and 
robustness of the labeling rule at each leaf node. We consider each of these in turn. 
	\subsection{Robustness Of Split Rules}
	\label{sec:split_rules}

As mentioned earlier, most decision tree algorithms select a split rule, $f$, by maximizing $C(f)$ defined by 
(\ref{eq:criterion}). Hence we are interested in comparing, for any specific $f$, the value of $C(f)$ with its  value, in the large sample limit, when labels are flipped under symmetric label noise. 

Let the noise-free samples at a node $v$ be $\{(\xx_i, y_i), \; i=1, \cdots, n\}$. Under label noise, the samples at this node would become $\{(\xx_i, \tilde{y}_i), \; i=1, \cdots, n\}$. Suppose in the noise-free case a split rule $f$ sends $n_l$ of these $n$ samples to the left child, $v_l$, and $n_r = n - n_l$ to right child, $v_r$. Note that a split rule is a function of only the feature vector. (For example, in an oblique decision tree the split rule could be: send a $\xx$ to left child if $\ww^T\xx + w_0 >0$). Since the split rule depends only on the feature vector $\xx$ and not the labels, the points that go to $v_l$ and $v_r$ would be the same for the noisy samples also. Thus, $n_l$ and $a = n_l/n$ would be same in both cases. However, what changes with label noise are the class labels on examples and   hence the number of examples of different classes at a node. 

Let $n^+$ and $n^- = n - n^+$ be the number of samples of the two classes at node $v$ in the noise-free case. Similarly, let $n_l^+$ and $n_l^- = n_l -n_l^+$ be the number of samples of the two classes at $v_l$ and define $n_r^+, n_r^-$ similarly. Let the corresponding quantities in the noisy case be $\tilde{n}^+, \tilde{n}^-, \tilde{n}_l^+, \tilde{n}_l^-$ etc. 
Define random variables, $Z_i, \; i=1, \cdots, n$ by $Z_i = 1$ if $\tilde{y}_i \neq y_i$ and $Z_i = 0$ otherwise. 
Thus, $Z_i$ are indicators of whether or not label on the $i^{th}$ example is corrupted. By definition of symmetric label noise, $Z_i$ are {\em iid} Bernoulli random variables with expectation $\eta$. 

Let $p = n^+/n, q= n^-/n = (1 - p)$ be the fractions of the two classes at $v$ under noise-free samples. Let $p_l, q_l$ and $p_r, q_r$ be these fractions for $v_l$ and $v_r$. Let the corresponding quantities for the noisy samples case be $\tilde{p}, \tilde{q}, \tilde{p}_l, \tilde{q}_l$ etc. Let $p^{\eta}, q^{\eta}$ be the values of $\tilde{p}, \tilde{q}$ in the large sample limit and similarly define $p^{\eta}_l, q^{\eta}_l$ ,$p^{\eta}_r, q^{\eta}_r$.

The value of $\tilde{n}^+$ is the number of $i$ such that $\tilde{y}_i = +1$. Similarly, the value of $n_l^{+}$ would be the number of $i$ such that $\xx_i$ is in $v_l$ and $\tilde{y}_i = +1$. Hence we have
\begin{equation}
\tilde{p} = \frac{\tilde{n}^+}{n} = \frac{1}{n} \left(\sum_{i: \tilde{y}_i = +1} \; 1\right)  = \frac{1}{n} \left(\sum_{i : y_i = +1} (1 - Z_i) \; + \; \sum_{i : y_i = -1} Z_i \right)
\label{eq:n-tilde-+}
\end{equation}
\begin{equation}
\tilde{p}_l = \frac{\tilde{n}^+_l}{n_l} = \frac{1}{n_l} \left( \sum_{i: \xx_i \in v_l, \tilde{y}_i = +1} \; 1 \right)  =  \frac{1}{n_l} \left(\sum_{i : \xx_i \in v_l, y_i = +1} (1 - Z_i) \; + \; \sum_{i : \xx_i \in v_l, y_i = -1} Z_i \right)
\label{eq:n-l-tilde-+}
\end{equation}

All the above expressions involve sums of independent random variables. Hence the values of the above quantities in the large sample limit can be calculated, by laws of large numbers, by essentially replacing each $Z_i$ by its expected value. Thus, from the above, we get
\begin{equation}
 p^{\eta}  = p(1 - \eta) + q \eta = p(1 - 2 \eta) + \eta; \ \  p_l^{\eta} = p_l(1 - \eta) + q_l \eta = p_l(1 - 2\eta) + \eta
\label{eq:p-eta}
\end{equation}

We emphasize here that, under symmetric label noise, the corruption of label is independent of feature vector and true label and thus we have 
$\Pr[Z_i = 1] = \Pr[Z_i=1 | y_i] = \Pr[Z_i =1 | \xx_i \in B, y_i] = \eta$, for any subset $B$ of the feature space. We have used this fact in deriving the eq.(\ref{eq:p-eta}). 
Comparing the expressions for $p^{\eta}$ and $p_l^{\eta}$, we see that, essentially, at any node (in the large sample limit) the fraction of examples whose labels are corrupted is the same. This is intuitively clear because under symmetric label noise the corruption of class label does not depend on the feature vector. 

To find the large sample limit of criterion $C(f)$ under label noise, we need values of the impurity function in the large sample limit which in turn needs $p^{\eta}, q^{\eta}, p^{\eta}_l$ etc. which are as given above. 
For example, the Gini impurity is given by $G(v) = 2p q$ for the noise free case. For the noisy sample, its value can be written as $\tilde{G}(v) = 2\tilde{p} \tilde{q}$. Its value in the large sample limit would be 
$G^{\eta}(v) = 2p^{\eta} q^{\eta}$. Another way  this can be seen is as
follows. Using eq.(\ref{eq:n-tilde-+}) one can show that
$E_{\eta}[\tilde{p} \tilde{q}] = p^{\eta}q^{\eta}-\frac{\eta (1 -
	\eta)}{n}$ which is $p^{\eta}q^{\eta}$ as $n$ goes to infinity.

Using the above we can now prove the following theorem about robustness of split criteria. 

	\begin{theorem}
		\label{theo_spli}
		Splitting criterion based on gini impurity, mis-classification rate and twoing rule are noise-tolerant (as per definition~1) to symmetric label noise given $\eta \neq 0.5$.
	\end{theorem}
	\begin{proof}\hfill
		
		As in the above, let $p$ and $q$ be the fractions of the two classes at $v$. For any split $f$, let $a$ be the fraction of points at the left child ($v_l$). Recall from above that  the fraction $a$ is same for noisy and noise-free data.\\
$\bullet$  {\bf Gini Impurity }
			For a node $v$, the gini impurity is $G_{\mbox{gini}}(v)=2pq$.
			Under symmetric label noise, gini impurity (under large sample limit) becomes (using eq.(\ref{eq:p-eta})), 
			\begin{flalign*}
			G_{\mbox{Gini}}^{\eta}(v)&=2p^{\eta}q^{\eta}=2[((1-2\eta)p+\eta)((1-2\eta)q+\eta)]\\
			&=2pq(1-2\eta)^2+(\eta-\eta^2)=G_{\mbox{Gini}}(v)(1-2\eta)^2+(\eta-\eta^2)
			\end{flalign*}	
Similar expressions hold for $G_{\mbox{Gini}}^{\eta}(v_l)$ and $G_{\mbox{Gini}}^{\eta}(v_r)$. 
 The (large sample) value of criterion or impurity gain  of $f$ under label noise can be written as
			\begin{flalign*}			
			\mbox{gain}^{\eta}_{\mbox{Gini}}(f)&=G_{\mbox{Gini}}^{\eta}(v)-[a\;G_{\mbox{Gini}}^{\eta}(v_l)+(1-a)G_{\mbox{Gini}}^{\eta}(v_r)]\\
			&=(1-2\eta)^2[G_{\mbox{Gini}}(v)-a\;G_{\mbox{Gini}}(v_l)-(1-a)\mbox{Gini}(v_r)]\\
			&=(1-2\eta)^2 \mbox{gain}_{\mbox{Gini}}(f)
			\end{flalign*}
Thus for any $\eta\neq 0.5$, if $\mbox{gain}_{\mbox{Gini}}(f^1)>\mbox{gain}_{\mbox{Gini}}(f^2)$, then $\mbox{gain}^{\eta}_{\mbox{Gini}}(f^1) >\mbox{gain}^{\eta}_{\mbox{Gini}}(f^2)$. 
			Which means that a maximizer of impurity gain based on gini index under noise-free samples will be also a maximizer of gain under symmetric label noise, under large sample limit. \\

			
			$\bullet$ {\bf Misclassification rate }
			For node $v$, misclassification impurity is, $G_{\mbox{MC}}(v)=\min\{p,q\}$.
			Under symmetric label noise with $\eta<0.5$, in the large sample limit, value of impurity is, 
(using eq.(\ref{eq:p-eta})), 
			\begin{flalign*}
			G_{\mbox{MC}}^{\eta}(v)&=\min\{p^{\eta},q^{\eta}\}=\min\{(1-2\eta)p+\eta, (1-2\eta)q+\eta\}\\
			&=(1-2\eta)G_{\mbox{MC}}(v)+\eta
			\end{flalign*}	
			In presence of symmetric label noise, expected impurity gain for a split $f$ can be written as
			\begin{flalign*}			
			\mbox{gain}_{\mbox{MC}}^{\eta}(f)&=G_{\mbox{MC}}^{\eta}(v)-[a\;G_{\mbox{MC}}^{\eta}(v_l)+(1-a)G_{\mbox{MC}}^{\eta}(v_r)]\\
			&=(1-2\eta)[G_{\mbox{MC}}^{\eta}(v)-a\;G_{\mbox{MC}}^{\eta}(v_l)-(1-a)G_{\mbox{MC}}^{\eta}(v_r)]\\
			&=(1-2\eta) \mbox{gain}_{\mbox{MC}}(f)
			\end{flalign*}  
			where $(1-2\eta)>0$ because we are considering the case $\eta < 0.5$. 
			When $\eta>0.5$, one can similarly show that $\mbox{gain}^{\eta}_{\mbox{MC}}(f)=(2\eta-1)\mbox{gain}_{\mbox{mc}}(f)$. This completes proof of noise-tolerance of impurity based on misclassification rate. \\
			
			$\bullet$ {\bf Twoing rule } 
			Using the same notation defined Sec~\ref{sec:crit_split} for
			twoing criterion, for a split $f$, objective can be rewritten as
			\begin{flalign*}	
			G_{\mbox{Twoing}}(f)&=\frac{a(1-a)}{4}\big{[}|p_l-p_r|+|q_l-q_r|\big{]}^2=a(1-a)[p_l-p_r]^2
			\end{flalign*}
			When there is symmetric label noise, $p^{\eta}_l=(1-2\eta)p_l+\eta$ and $p_r^{\eta}=(1-2\eta)p_r+\eta$.
			\begin{flalign*}
			G_{\mbox{Twoing}}^{\eta}(f)&=a(1-a)[p_l^{\eta}-p_r^{\eta}]^2=a(1-a)(1-2\eta)^2[p_l-p_r]^2\\
			&=(1-2\eta)^2 G_{\mbox{Twoing}}(f)
			\end{flalign*}
			Thus, the maximizer of twoing rule does not change when there is symmetric label noise. 
	\end{proof}
The above theorem shows that impurity gain (using gini or misclassification rate) based criteria are noise-tolerant for symmetric label noise as per Definition~1.  
	
	\begin{remark} {\bf Impurity based on entropy}
Another popular criterion is impurity gain based on entropy which is not considered in the above theorem. 
The impurity gain based on entropy is not noise-tolerant as per definition~1 as shown by the following counterexample.

Consider a case where a node has $n$ samples ($n$ is large). Suppose, under split rule $f_1$ we get $n_l = n_r = 0.5 n$, $n_l^+ = 0.05 n$ and $n_r^+ = 0.25 n$. Suppose there is another split rule $f_2$ under which we get $n_l = 0.3 n$ and $n_r = 0.7n$ with 
$n_l^+=0.003 n$ and $n_r^+ = 0.297n$. Then it can be easily shown that $\mbox{gain}_{\mbox{Entropy}}(f_1)<\mbox{gain}_{\mbox{Entropy}}(f_2)$; but, under symmetric label noise with $\eta=40\%$, $\mbox{gain}^{\eta}_{\mbox{Entropy}}(f_2)<\mbox{gain}^{\eta}_{\mbox{Entropy}}(f_1)$.

However, we would like to emphasize that the above example may be a non-generic one. 
In large number of simulations we have seen that the split rule that maximizes the criterion is same under noisy and noise-free cases. Thus, impurity gain based on entropy for learning decision trees is also fairly robust to label noise.
	\end{remark}
\subsection{Robustness of Labeling Rule at Leaf Nodes }
	We next consider the robustness of criterion to assign a class label to a leaf node.
	A popular approach is to take majority vote at the leaf node. We prove that, majority voting is robust to symmetric label noise in the sense that (in the large sample limit) the fraction of positive examples would be more under label noise if the fraction of positive examples is higher in noise-free case. We also show that it can be robust to non-uniform noise also under a restrictive condition.  
	\begin{theorem}
		\label{theo_mojority}
		Let $\eta_{\xx} < 0.5,\forall \xx$. (a). Then, majority voting at a leaf node is robust to symmetric label noise. (b). It is also robust to nonuniform label noise if all the points at the leaf node belong to one class in the noise free data.
	\end{theorem}
	\begin{proof}\hfill 
		
		Let $p$ and $q=1-p$  be the fraction of positive and negative samples at leaf node $v$.		
		
		(a) Under symmetric label noise, the relevant fractions are
$p^{\eta}=(1-\eta)p+\eta q$ and $q^{\eta}=(1-\eta)q+\eta p$. Thus, $p^{\eta}-q^{\eta}=(1-2\eta)(p-q)$. Since $\eta<0.5$, $(p^{\eta}-q^{\eta})$ will have the same sign as $(p-q)$, proving robustness of the majority voting.
			
			(b)	Let $v$ contain all the points from the positive class. Thus, $p=1, q=0$. Let $\xx_1, \cdots, \xx_n$ be the samples at $v$. Under non-uniform noise (with $\eta_{\xx}<0.5, \forall \xx$), 
			\begin{flalign}
			p^{\eta}=\frac{1}{n}\sum_{i=1}^{n} (1-\eta_{\xx_i})> \frac{0.5}{n}\sum_{i=1}^{n} 1 = 0.5
			\end{flalign}
			Thus, the majority vote will assign positive label to the leaf node $v$. This proves the second part of the theorem.  
	\end{proof}		 
		\subsection{Robustness of Decision Tree Learning Under Symmetric Label Noise : Large Sample Analysis}
		\label{dt_robust} 
		
		We have proved that some of the popular split criteria are noise-tolerant. What we have shown is that the split rule that maximizes the criterion under noise-free samples is same as that which maximizes the value of criterion under symmetric label noise (under large sample limit). 
This means, under large sample assumption, 
the same split rule would be learnt at any node irrespective of whether the labels come from noise-free data or noisy data. (Here we assume for simplicity that there is a unique split rule maximizing the criterion at each node. Otherwise we need some prefixed rule to break ties).\footnote{Here we are assuming that the $\xx_i$ at the node are same in the noisy and noise-free cases. These are same at the root. If in the two cases we learn the same split at the root, then at both its children the samples would be same in the noise and noise-free cases and so on.} 

Our result for leaf node labeling implies that, under large sample assumption, with majority rule a leaf node would get the same label under noisy or noise-free data. To conclude that we learn the same tree, we need to examine the rule for deciding when a node becomes a leaf. If this is determined by the depth of the node or number of samples at the node then it is easy to see that the same tree would be learnt with noisy and noise-free data. In many algorithms one makes a node as leaf if no split rule gives positive value to the gain. This will also lead to learning of the same tree with noisy samples as with noise-free samples, because we showed that the gain under noisy case is a linear function of the gain under noise-free case. 
\begin{remark}{\bf Robustness under general noise:}
In our analysis so far, we have only considered symmetric label noise. In the simplest case of asymmetric noise, namely, class-conditional noise, noise rate is same for all feature vectors of a class though it may be different for different classes. In the risk minimization framework, class conditional noise can be taken care when the noise rates are known (or can be estimated) \citep{natarajan2013learning,DBLP:conf/colt/ScottBH13,ghosh2015making}. We can extend the analysis presented in Sec.3.1 to relate expected fraction of examples of a class in the noisy and noise-free cases using the two noise rates. Thus, if the noise rates are assumed known (or can be reliably estimated) it should be possible to extend the analysis here to the case of class-conditional noise. 
In the general case when noise rates are not known (and cannot be reliably estimated), it appears difficult to establish robustness of impurity based split criteria.
\end{remark}
	
	\subsection{Sample Complexity under Noise}
	We established robustness of decision tree learning algorithms under large sample limit. Hence an interesting question is that of how large the sample size should be for our assertions about robustness 
to hold with a large probability. We provide some sample complexity bounds in this subsection. (Proofs of Lemmas \ref{sample_leaf} and \ref{sample-node} are given in Appendix).  
	\begin{lemma}
		\label{sample_leaf}
		Let leaf node $v$ have $n$ samples.  Under symmetric label noise with $\eta <0.5$, majority voting will not fail with probability at least  $1-\delta$ when $n\geq \frac{2}{\rho^2(1-2\eta)^2}\ln(\frac{1}{\delta})$, where $\rho$ is the difference between fraction of positive and negative samples in the noise-free case.
	\end{lemma}
	
	The sample size needed increases with increasing $\eta$, which is intuitive. It also increases with decreasing $\rho$. The value of $\rho$ tells us the `margin of majority' in the noise-free case and hence when $\rho$ is small we should expect to need more examples in the noisy case.  				
	\begin{lemma}	
\label{sample-node}
		Let there be $n$ samples at a non-leaf node $v$ and given two splits $f_1$ and $f_2$, suppose gain (gini, misclassification, twoing rule) for $f_1$ is higher than that for $f_2$. Under symmetric label noise with $\eta\neq 0.5$, gain from $f_1$ will be higher with probability $1-\delta$ when $n\geq \mathcal{O}(\frac{1}{\rho ^2(1-2\eta)^2}\ln (\frac{1}{\delta}))$, where $\rho$ denotes the difference between gain of the two splits in the noise-free case.
	\end{lemma}
	
	While these results, shed some lights on sample complexity, we emphasize that these bounds are loose and are obtained using concentration inequalities. Also we want to point out, large sample in leaf implies large sample in non-leaf nodes. In practice, sample size needed is not high. In experimental section, we provide results on how many training samples are needed for robust learning of decision trees on a synthetic dataset.  
	
	\subsection{Noise Robustness in Random Forest} A random forest \citep{breiman2001random} is a collection of randomized tree classifiers. We represent the set of trees as $g_n=\{g_n(\xx,\pi_1),\cdots,g_n(\xx,\pi_m)\}$. Here $\pi_1,\cdots,\pi_m$ are {\em iid} random variables, conditioned on data, which are used for partitioning the nodes. Finally, majority vote is taken among the random tree classifiers for prediction. We denote this classifier as $\bar{g}_n$.
	
	In a {\bf\em purely random forest classifier}, partitioning does not depend on the class labels. At each step, a node is chosen randomly and a feature is selected randomly for the split. A split threshold is chosen uniformly randomly from the interval of the selected feature. This procedure is done $k$ times.
	{\bf\em A greedily grown random forest classifier} is a set of randomized tree classifiers. 
Each tree is grown greedily by improving impurity with some randomization. At each node, a random subset of features are chosen. Tree is grown by computing the best split among those random features only. Breiman's random forest classifier uses gini impurity gain \citep{breiman2001random}.
	\begin{remark}
		A purely random forest classifier/ greedily grown random forest, $\bar{g}_n$, is robust to symmetric label noise with $\eta<0.5$ under large sample assumption.
		
		We need to show each randomized tree is robust to label noise in both cases. 
		In purely random forest, randomization is on the partitions and the partitions do not depend on class labels (which may be noisy). 
We proved robustness of majority vote at leaf nodes under symmetric label noise. Thus, for a purely random forest, $\bar{g}^{\ast\eta}=\bar{g}^{\ast}$. That is, the classifier learnt with noisy labels would be same as that learnt with noise-free samples.
		Similarly for a greedily grown trees with gini impurity measure, we showed that each tree is robust because of both split rule robustness and majority voting robustness. Thus when large sample assumption holds, greedily grown random forest will also be robust to symmetric label noise.
	\end{remark}
	\begin{remark}{\bf Sample complexity of Random forest:}
		\label{remark:sample_rf}
Empirically we observe that, often random forest has better robustness than a single decision tree in finite sample cases.		
		For a classifier, generalization error can be written as, \[\mbox{$error_{gen}$}=\mbox{$error_{bias}$}+\mbox{$err_{variance}$}+\sigma^2_{noise}\]
		Under symmetric label noise, $error_{bias}$ is same for single decision tree as well as random forest. Thus generalization error is controlled by $error_{variance}$. 
If pairwise correlation of each trees is $\rho$ and variance is $\sigma^2$ for each tree, then random forest, consisting $N$ trees, has variance, \citep{hastie2005elements}
		\[error_{variance}=\rho\sigma^2+\frac{1-\rho}{N}\sigma^2\]  
		Intuitively, if a single decision tree is learnt with noisy samples, our results imply that its classification decision on a new point would be same as noise free case in an expected sense. 
If we have many independent decision trees, variance in the classification will decrease. If the decision trees are highly correlated, then the variance reduction might not be significant.
	\end{remark}
	
\section{Empirical Illustration}

	In this section, we illustrate our robustness results for learning of decision trees and random forest. We also present results with SVM. While, SVM has been proved to be non-robust even under symmetric label noise, its sensitivity towards noise widely varies \citep{long2010random,nettleton2010study,manwani2013noise,van2015learning}.
We also provide results on sample complexity for robust learning of decision trees and random forest. 
	\subsection{Dataset Description}
	We used four 2D synthetic datasets. Details are given below. (Here $n$ denotes total number of samples, $p_+, p_-$ represent the class conditional densities, and $\mathcal{U}(\mathcal{A})$ denotes uniform distribution over set $\mathcal{A}$).  
\begin{itemize}
	\item Dataset 1: Checker board 2by2 Pattern: Data uniform over $[0, 2] \times [0, 2]$ and one class region being $([0,1] \times [0, 1]) \cup ([1, 2] \times [1, 2])$ and $n = 30000$
	\item Dataset 2: Checker board 4by4 Pattern: Extension of the above to a $4 \times 4$ grid.  
	\item Dataset 3: Imbalance Linear Data. $p_{+}=\mathcal{U}([0,0.5]\times [0,1])$ and  $p_{-}=\mathcal{U}([0.5,1]\times [0,1])$. Prior probabilities of classes are $0.9$ \& $0.1$, and  $n=40000$. 
	\item Dataset 4: Imbalance and Asymmetric Linear Data. $p_{+}=\mathcal{U}([0,0.5]\times [0,1])$ and  $p_{-}=\mathcal{U}([0.5,0.7]\times [0.4,0.6])$. Prior probabilities are $0.8$ \& $0.2$,  and $n=40000$. 
\end{itemize}
We also present results for $6$ UCI datasets \citep{Lichman:2013}. 
	\subsection{Experimental Setup}
	We used decision tree implementation in scikit learn library \citep{pedregosa2011scikit}. We present results only with gini impurity based decision tree classifier. (We observed that decision trees learnt using twoing rule and misclassification rate have similar performance). 
For random forest classifier (RF) we used scikit learn library. Number of trees in random forest was set to 100. For SVM we used libsvm package \citep{CC01a}. 

In subsection~\ref{exp_sens} we present results to illustrate sample complexity for robust learning where training set size and size of leaf nodes is varied as explained there. 	
	 
In subsection~\ref{exp_sim}, we compare accuracies of decision tree learning, random forest and SVM for which the following setup is used. 
Minimum leaf size is the only user-chosen parameter in random forest and decision trees. For synthetic datasets, minimum samples in leaf node was restricted to 250. For UCI datasets, it was restricted to 50. For SVM, we used linear kernel (l) for Synthetic Datasets 3, 4 and quadratic kernel (p) for Checker board 2by2 data.
In all other datasets we used gaussian kernel (g).
	For SVM, we selected hyper-parameters using validation data. (Validation range for $C$ is $0.01$-$500$ and for $\gamma$ in the Gaussian kernel it is $0.001$-$10$). We used $20\%$ data for testing and $20\%$ for validation. 
	Symmetric label noise was varied from $0\%-40\%$. 
	As synthetic datasets are separable, we also experimented with class conditional noise with the two noise rates for the two classes being $40\%$ and $20\%$. In all experiments, noise was introduced only on training and validation data. Test set was noise free. 

\subsection{Effect of sample size on robustness of learning}
\label{exp_sens}
Here we discuss sensitivity of decision tree learning (under label noise) on sample size. 
We present experimental results on the test accuracy for different sample sizes using the 2by2 checker board data.

To study effect of sample size in leaf nodes, we choose a leaf sample size and learn decision tree and random forest with 
different noise levels. (The training set size is fixed at 20000). We do this for a number of choices for leaf sample size. The test accuracies in all these cases are shown in Figure~\ref{fig:sens}(a). As can be seen from the figure, even when training data size is huge, we do not get robustness if leaf sample size is small. This is in accordance with our analysis (as in Lemma~\ref{sample_leaf}) because minimum sample size is needed for the majority rule to be correct with a large probability. A leaf sample size of 50 seems sufficient to take care of even 30\% noise. As expected, random forest has better robustness. 
	\begin{figure}
		\begin{center}
			\begin{tabular}{cc}
				\includegraphics[scale=.5]{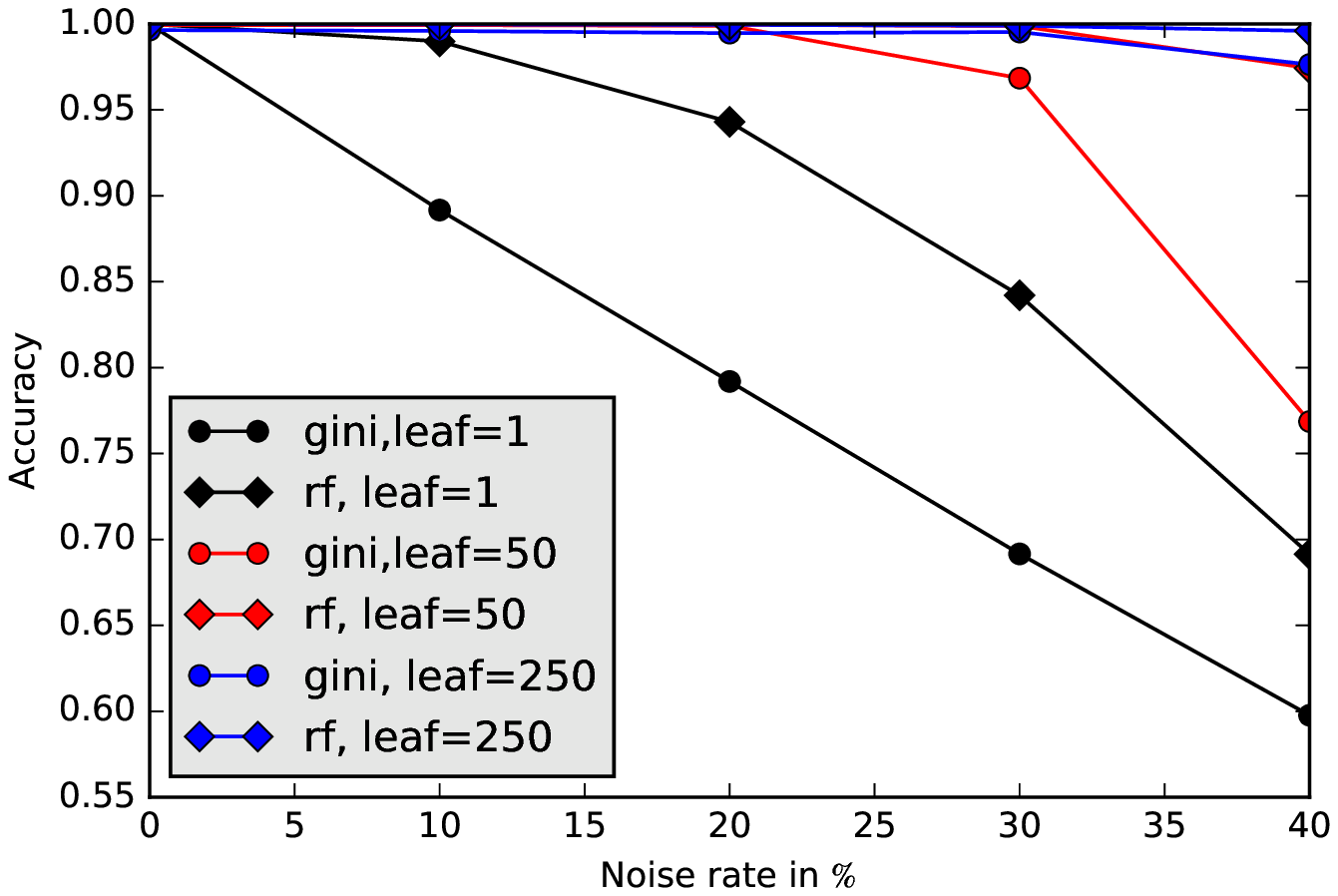}  &  \includegraphics[scale=.5]{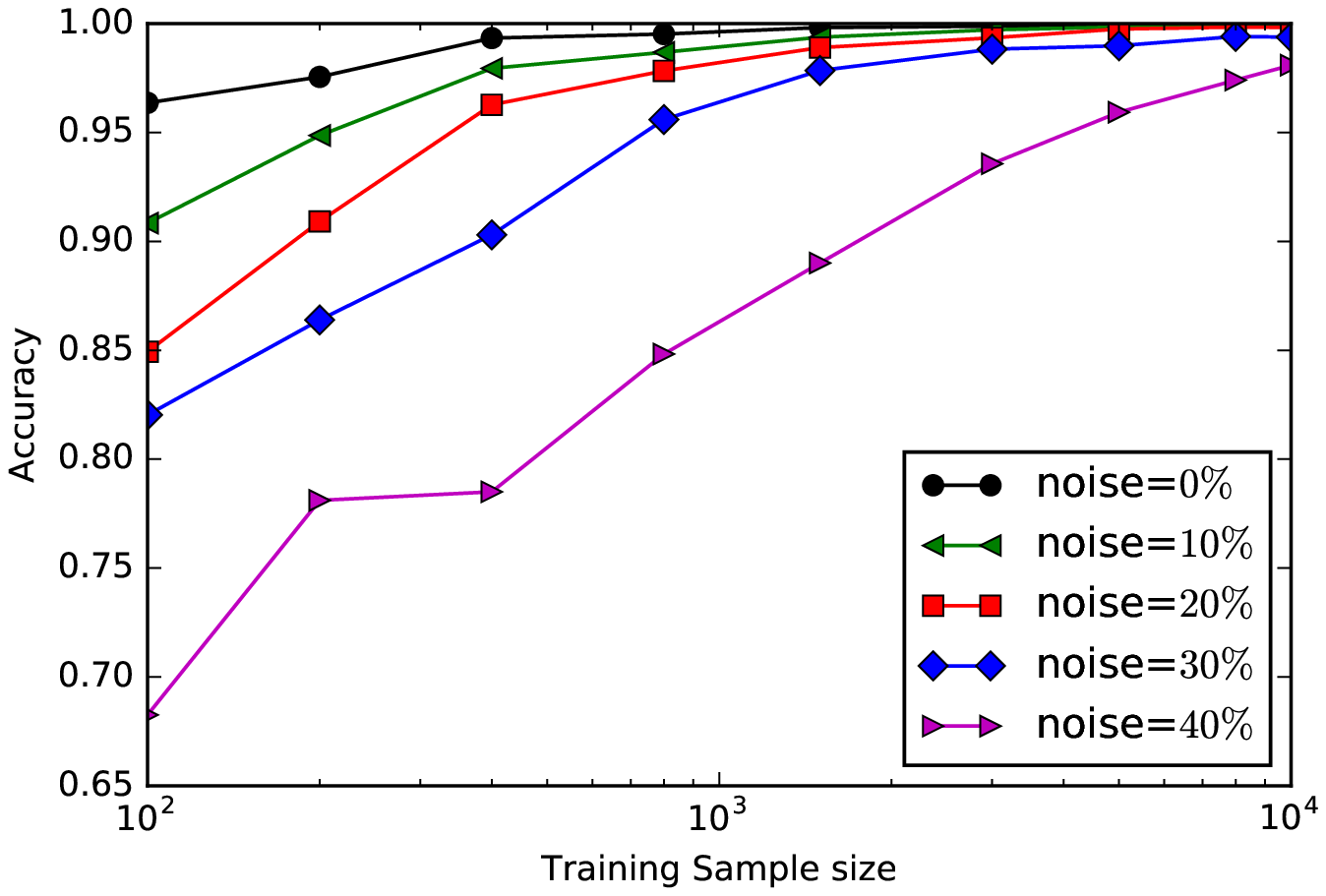}\\
				(a) & (b) \\
			\end{tabular}
			\caption{\footnotesize{For 2by2 Checker Board data (a) Minimum leaf size varied from 1 to 250 for both RF and DT,
					(b)Training Data size varied from $100$ to $10000$ for different noise levels for DT }}
			\label{fig:sens}
		\end{center}
	\end{figure}

Next we experiment with varying the (noisy) training data size. The results are shown in Figure~\ref{fig:sens}(b). It can be seen that with 400/4000 sample size decision tree learnt has good test accuracy ($95\%$) at $20\%/40\%$ noise (the sample ratio is close to $\frac{(1-2\times 0.4)^2}{(1-2\times 0.2)^2}=1/9$ as provided in lemma.~\ref{sample_leaf}). We need larger sample size for higher level of noise.  This is also as expected from our analysis. 
	\subsection{Comparison of accuracies of learnt classifiers}
	\label{exp_sim} 
 The average test accuracy and standard deviation (over 10 runs) on different data sets under different levels of noise are shown in Table~\ref{tab-2} for synthetic datasets and in Table~\ref{tab-3} for UCI datasets. In table~\ref{tab-3} we also indicate the dimension of feature vector ($d$), the number of positive and negative samples in the data ($n^+, n^-$). 
	
	For synthetic datasets the sample sizes are large and hence we expect good robustness. As can be seen from Table~\ref{tab-2}, for noise-free data, decision tree, random forest and SVM have all similar accuracies. However, with 30\% or 40\% noise, the accuracies of SVM are much poorer than those of decision tree and random forest. For example on datasets 3 and 4, the accuracies of decision tree and random forest continue to be 99\% even at 40\% noise while those of SVM drop to about 90\% and 80\% respectively. This illustrates the robustness of decision tree learning as indicated by our analysis. It can be seen that decision tree and random forest are robust to class conditional noise also, even without knowledge about noise rate (as indicated by last column in the table). Our current analysis does not prove this robustness; as remarked earlier, this is one possible extension of the theoretical analysis presented here.


Similar performance is seen on UCI data sets also as shown in Table~\ref{tab-3}. 	
	For breast cancer dataset, accuracy of decision tree also drops with noise while for random forest the drop is significantly less. This is also expected because the total sample size here is less.  Although SVM has significantly higher accuracy than decision tree in 0\% noise, at 40\% noise its accuracy drops more than that of decision tree. In all other data sets also, decision tree and random forest are more robust than SVM as can be seen from the table. 	

As explained earlier, our analysis shows that decision tree learning is robust in large sample case. Thus, though decision tree learning may not be robust to label noise when training set size is small, the robustness improves with increasing training set size. This is demonstrated by our results on synthetic data sets. However, this is not true of a standard algorithm such as SVM. 	For example, Datasets 3 and 4 represent  very simple two dimensional problems. Though we have 40000 samples here, SVM does not learn well under label noise. On the other hand, the accuracies of decision tree and random forest  at 30\% noise are as good as their accuracies at 0\% noise and these accuracies are very high. 

	
	
\begin{table}[t]
	\begin{center}
		\caption{Comparison of Accuracies on Synthetic datasets}
		\label{tab-2}
			\resizebox{\columnwidth}{!}{%
		\begin{tabular}{|p{1.1cm}|p{1.4cm}|p{2cm}|p{2cm}|p{2cm}|p{2cm}|p{2cm}|p{2cm}|p{2cm}|}
			\hline
			\hline
			Data & Method & $\eta=0\%$ & $\eta=10\%$ & $\eta=20\%$ & $\eta=30\%$ & $\eta=40\%$ & $\eta_{+}=40\%$ $\eta_{-}=20\%$  \\\hline
			\multirow{3}{*}{\parbox{1cm}{2$\times$2 CB}} & Gini &99.95 $\pm$0.05 & 99.9 $\pm$0.06 & 99.91 $\pm$0.1 & 99.82 $\pm$0.16 & 98.97 $\pm$0.83 & 99.45 $\pm$0.83\\
			& RF &99.99 $\pm$0.02 & 99.96 $\pm$0.02 & 99.91 $\pm$0.05 & 99.87 $\pm$0.06 & 99.16 $\pm$0.18& 99.11 $\pm$0.45\\
			& SVM(p) &99.83 $\pm$0.12 & 97.38 $\pm$1.21 & 91.88 $\pm$2.65 & 87.96 $\pm$5.52 & 76.42 $\pm$4.43 & 68.78 $\pm$0.97 \\\hline
			\multirow{3}{*}{\parbox{0.5cm}{4$\times$4 CB}} &Gini&99.76 $\pm$0.18 & 99.72 $\pm$0.16 & 99.46 $\pm$0.18 & 98.71 $\pm$0.32 & 95.21 $\pm$1.08 & 97.36 $\pm$1.23\\
			& RF & 99.94 $\pm$0.02 & 99.9 $\pm$0.02 & 99.78 $\pm$0.04 & 99.35 $\pm$0.15 & 96.23 $\pm$0.91 & 95.41 $\pm$0.53\\
			& SVM(g) &99.6 $\pm$0.05 & 98.58 $\pm$0.23 & 97.81 $\pm$0.24 & 96.83 $\pm$0.46 & 92.22 $\pm$2.5 & 91.24 $\pm$0.85\\\hline
			\multirow{3}{*}{\parbox{0.5cm}{Dataset 3}} & Gini &100.0 $\pm$0.01 & 100.0 $\pm$0.01 & 99.99 $\pm$0.01 & 99.99 $\pm$0.02 & 99.92 $\pm$0.07 & 99.92 $\pm$0.18
			\\
			& RF& 100.0 $\pm$0.01 & 100.0 $\pm$0.01 & 99.99 $\pm$0.01 & 99.98 $\pm$0.02 & 99.86 $\pm$0.12 & 99.9 $\pm$0.13\\
			& SVM (l) & 99.89 $\pm$0.04 & 96.65 $\pm$0.26 & 90.02 $\pm$0.3 & 90.02 $\pm$0.3 & 90.02 $\pm$0.3 &  90.1 $\pm$0.31\\\hline
			\multirow{3}{*}{\parbox{1.5cm}{Dataset 4}} & Gini &100.0 $\pm$0.0 & 99.99 $\pm$0.01 & 99.99 $\pm$0.01 & 99.98 $\pm$0.03 & 99.73 $\pm$0.54 & 99.88 $\pm$0.26\\
			& RF & 100.0 $\pm$0.0 & 99.99 $\pm$0.01 & 99.99 $\pm$0.01 & 99.93 $\pm$0.09 & 99.91 $\pm$0.11 & 99.7 $\pm$0.31\\
			& SVM(l) & 99.86 $\pm$0.03 & 99.21 $\pm$0.24 & 96.55 $\pm$4.05 & 79.96 $\pm$0.34 & 79.96 $\pm$0.34 & 79.96 $\pm$0.34\\\hline
			
		\end{tabular}}
	\end{center}
\end{table}
\begin{table}[t]
	\centering
		\caption{Comparison of Accuracies on UCI datasets}
		\label{tab-3}
		\resizebox{\columnwidth}{!}{%
		\begin{tabular}{|p{2.55cm}|p{1.2cm}|p{2.1cm}|p{2.1cm}|p{2.1cm}|p{2.1cm}|p{2.1cm}|p{2.1cm}|}
			\hline
			\hline
			Data $(d, n^{+}, n^{-})$ & Method & $\eta=0\%$ & $\eta=10\%$ & $\eta=20\%$ & $\eta=30\%$ & $\eta=40\%$ \\\hline
			\multirow{3}{*}{\parbox{2.4cm}{Breast Cancer  $(10,239,444)$}} & Gini &92.37 $\pm$1.9 & 92.59 $\pm$2.62 & 90.47 $\pm$3.08 & 90.58 $\pm$2.76 & 83.65 $\pm$7.36 \\ 
			& RF & 96.06 $\pm$1.41 & 96.02 $\pm$1.94 & 96.31 $\pm$1.95 & 94.74 $\pm$3.54 & 91.93 $\pm$4.86\\
			& SVM & 96.35 $\pm$1.46 & 95.58 $\pm$2.11 & 95.26 $\pm$2.63 & 92.81 $\pm$3.22 & 83.47 $\pm$13.2\\ \hline		
			
			\multirow{3}{*}{\parbox{1.5cm}{German $(24,300,700)$}} & Gini &72.05 $\pm$4.89 & 69.4 $\pm$4.04 & 71.95 $\pm$2.72 & 68.6 $\pm$3.44 & 65.25 $\pm$6.71\\ 
			& RF & 69.0 $\pm$3.33 & 69.1 $\pm$3.45 & 69.3 $\pm$3.24 & 69.15 $\pm$3.47 & 69.15 $\pm$4.55\\
			& SVM & 75.15 $\pm$3.26 & 71.95 $\pm$2.89 & 72.25 $\pm$4.39 & 66.4 $\pm$4.84 & 60.9 $\pm$8.43\\ \hline
			\multirow{3}{*}{\parbox{1.5cm}{Splice $(60,1648, 1527)$}} & Gini &91.39 $\pm$1.04 & 91.31 $\pm$0.7 & 89.84 $\pm$1.79 & 85.67 $\pm$2.99 & 73.56 $\pm$8.13\\ 
			& RF & 94.57 $\pm$1.47 & 93.87 $\pm$0.91 & 92.98 $\pm$1.4 & 91.84 $\pm$1.12 & 81.92 $\pm$4.04\\
			& SVM & 90.93 $\pm$1.4 & 88.98 $\pm$0.92 & 86.14 $\pm$1.47 & 81.42 $\pm$1.49 & 67.21 $\pm$6.63\\ \hline
			\multirow{3}{*}{\parbox{1.5cm}{Spam $(57,1813,2788)$}} & Gini &88.99 $\pm$1.45 & 89.02 $\pm$1.04 & 87.39 $\pm$2.04 & 84.06 $\pm$3.26 & 79.59 $\pm$3.72\\ 
			& RF & 91.8 $\pm$1.27 & 91.9 $\pm$1.07 & 91.52 $\pm$1.07 & 91.68 $\pm$1.22 & 88.71 $\pm$3.19\\
			& SVM & 89.72 $\pm$1.07 & 86.18 $\pm$1.35 & 83.43 $\pm$1.47 & 77.45 $\pm$2.38 & 69.23 $\pm$3.05\\ \hline
			\multirow{3}{*}{\parbox{2.3cm}{Wine (white) $(11,3258,1640)$}} & Gini &75.44 $\pm$0.98 & 74.31 $\pm$1.43 & 74.64 $\pm$1.4 & 73.58 $\pm$1.46 & 66.64 $\pm$5.09 \\ 
			& RF & 76.58 $\pm$0.8 & 76.17 $\pm$0.96 & 76.23 $\pm$1.25 & 75.51 $\pm$1.52 & 71.14 $\pm$2.37\\
			& SVM &75.62 $\pm$0.7 & 74.39 $\pm$1.3 & 71.64 $\pm$2.23 & 68.52 $\pm$2.53 & 61.54 $\pm$5.25\\ \hline
			\multirow{3}{*}{\parbox{2.3cm}{Magic $(10,12332,6688)$}} & Gini &84.06 $\pm$0.59 & 83.91 $\pm$0.67 & 83.0 $\pm$0.62 & 81.88 $\pm$0.64 & 78.25 $\pm$1.79\\ 
			& RF & 85.81 $\pm$0.25 & 85.79 $\pm$0.43 & 85.64 $\pm$0.37 & 85.26 $\pm$0.44 & 82.72 $\pm$1.18\\
			& SVM & 82.98 $\pm$0.47 & 82.4 $\pm$0.32 & 81.54 $\pm$0.35 & 79.73 $\pm$0.4 & 71.53 $\pm$2.56\\ \hline
			
		\end{tabular}}
\end{table}

	\section{Conclusion}
	
	In this paper, we investigated the robustness of decision tree learning under label noise. In many current applications one needs to take care of label noise in training data. Hence, it is very desirable to have learning algorithms that are not affected by label noise. Since most impurity based top-down decision tree algorithms learns split rules based on fractions of positive and negative samples at a node, one can expect that they should have some robustness. We proved that decision tree algorithms based on gini or misclassification impurity and the twoing rule algorithm are all robust to symmetric label noise. We showed that, under large sample assumption, with a high probability, the same tree would be learnt with noise-free data as with noisy data. We also provided some sample complexity results for the robustness. Through extensive empirical investigations we illustrated the robust learning  of decision tree and  random forest.
	Decision tree approach is very popular in many practical applications. Hence, the robustness results presented in this paper are interesting. All the results we proved are for symmetric noise. Extending these results to class conditional and non-uniform noise is an important direction for future research. 
	\bibliography{acml16}
	
	\appendix
	
	\section{Sample Complexity Bounds}\label{apd:first}
	\begin{proof}[of Lemma 7]
		\hfill		
		
		Let $n^+$ and $n^-$ denote the positive and negative samples at the node under noise-free case. (Note $n= n^+ + n^-$). Without loss of generality assume that positive class is in majority and hence, by definition, $\rho = (n^+ - n^-)/n$. 
Let $\tilde{n}^+$ and $\tilde{n}^-$ be the positive and negative samples under the noisy case. 
		
		Let $X_i$, $i=1, \cdots, n^+$ be random variables with $\Pr[X_i =1] =  1 - \Pr[X_i =0] = \eta$. 
		Let $X_i$, $i=n^+ +1, \cdots, n$ be random variables with $\Pr[X_i = -1] =  1 - \Pr[X_i =0] = \eta$.
		Let $S_n = \sum_{i=1}^n X_i$. Then, under symmetric label noise, we have $\tilde{n}_p - \tilde{n}_n = (n_p - n_n) - 2 S_n = \rho n - 2 S_n$. Also, note that $ES_n = \eta n^+ - \eta n^- = \eta \rho n$. Now we have
		\begin{eqnarray*}
			\Pr[ \tilde{n}^+ - \tilde{n}^- < 0 ] & = & \Pr[\rho n - 2 S_n < 0 ] \\
			& = &    \Pr[2 S_n - 2 ES_n > \rho n (1 - 2 \eta) ] \\
			& \leq & \exp\left(-\frac{\rho^2 n (1 - 2 \eta)^2}{2}\right)
		\end{eqnarray*}
		where the last line follows from hoeffding's inequality. If we want this probability to be less than $\delta$ then we would need $n > \frac{2}{\rho^2 (1 - 2 \eta)^2} \ln(\frac{1}{\delta})$. This completes the proof. 
	\end{proof}
	
	\begin{proof}[Of Lemma 8]
		\hfill  
		
		Lets assume parent node $v$ contains $n$ samples whereas left child $v_l$ (right child $v_r$) contains $n_l=na$ ($n_r=n-na$) samples. 	Note under noise, for a split rule $f$ at node $v$, for both parent as well as child, these numbers remain same as noise free case. For a parent node $v$, suppose, $p$ ($\tilde{p}$) and $q$ ($\tilde{q}$) are the positive and negative fraction under noise-free (noisy) data with $n$ samples. Similarly $p_l$, $q_l$, $\tilde{p}_l$,  $\tilde{q}_l$ ($p_r$, $q_r$, $\tilde{p}_r$,  $\tilde{q}_r$) is defined for left (and right child).
		Thus under symmetric label noise $\eta$, we can write for any node (note, $E_{\eta}(\tilde{p})=p^{\eta}$),
		\begin{equation}
			\label{eq:hoeff}
			\Pr[|\tilde{p}-p^{\eta}|> \epsilon]\leq 2e^{-2n\epsilon^2}
		\end{equation}
		We want to bound how finite sample estimates of different impurity gain differs from the large sample assumption (or the expectation). 
		We use $\epsilon_1$, $\epsilon_2$ and $\epsilon_3$ to denote the finite sample error (from the expectation) for positive fraction in parent, left and right child respectively (note this in turn bounds negative fraction also). 
		We set $\epsilon_1=\epsilon$, $\epsilon_2=\epsilon/\sqrt{a}$ and  $\epsilon_3=\epsilon/\sqrt{1-a}$. The probability can be upper bounded using hoeffding bound in eq. (\ref{eq:hoeff}) as,
		\begin{equation}
			\Pr\Big{[}\big{(}|\tilde{p}-p^{\eta}|\geq\epsilon_1 \big{)} \cup \big{(}|\tilde{p}_l-p_l^{\eta}|\geq \epsilon_2\big{)} \cup \big{(}|\tilde{p}_r-p^{\eta}_r|\geq \epsilon_3\big{)}\Big{]}\leq 2(e^{-2n\epsilon_1^2}+e^{-2n_l\epsilon_2^2}+e^{-2n_r\epsilon_3^2})=6e^{-2n\epsilon^2}
			\label{eq:prob_union}
		\end{equation}
		Note that, this probability does not depend on any split and can be applied to any arbitrary split. Also note, for twoing rule, first term is not required in RHS and LHS.
		Given the complement of this event (lets call it as 
		\lq all fractions are $\epsilon$-accurate\rq$\;$event), we compute how finite sample impurity gain deviates from the large sample limit.\\
		
		$\bullet$ Gini Impurity:
		 For a node $v$, after some simplification, using eq.~\ref{eq:hoeff},\ref{eq:prob_union}, we can bound the finite sample noise estimate as (for gini $\tilde{G}=2\tilde{p}\tilde{q})$, \[|\tilde{p}\tilde{q}-p^{\eta}q^{\eta}|\leq
		|\epsilon (p^{\eta}-q^{\eta})|\]
		Thus we can bound finite noisy sample gain from gini impurity as,
		\begin{flalign*}
			|\hat{\mbox{gain}}_{\mbox{Gini}}^{\eta}(f)-\mbox{gain}_{\mbox{Gini}}^{\eta}(f)|&\leq 2|\epsilon_1(p^{\eta}-q^{\eta})|+2a|\epsilon_2 (p_l^{\eta}-q_l^{\eta})|+2(1-a)|\epsilon_3(p_r^{\eta}-q_r^{\eta})|\\
			&\leq 2(1-2\eta)\Big{[} |\epsilon_1(p-q)|+a|\epsilon_2(p_l-q_l)|+(1-a)|\epsilon_3(p_r-q_r)|\Big{]}\\
			& \leq 2(1-2\eta)[|\epsilon_1(p-q)|+|a\epsilon_2|+|(1-a)\epsilon_3|]
			\leq 6(1-2\eta)\epsilon
		\end{flalign*}	
		
		Under noise free case, we assume the difference of gini gain between two splits is $\rho$. Under noise corrupted signal label, expected difference is $\rho^{\eta}=(1-2\eta)^2\rho$. 
		
		Setting $\epsilon=\rho^{\eta}/12(1-2\eta)=\rho (1-2\eta)/12$ for both the splits in eq.~\ref{eq:prob_union}, we get the upper bound on probability of ordering change as,
		$12e^{-n\rho^2(1-2\eta)^2/72}$.\\
		
		$\bullet${Misclassification Impurity:}
		For misclassification impurity, for a node $v$, we have \[|\min(\tilde{p},\tilde{q})-\min(p^{\eta},q^{\eta})|\leq
		|\epsilon|\] 
		Thus we can bound finite noisy sample gain for misclassification impurity as,
		\begin{flalign*}
			|\hat{\mbox{gain}}_{\mbox{MC}}^{\eta}(f)-\mbox{gain}_{\mbox{MC}}^{\eta}(f)|&\leq |\epsilon_1|+a|\epsilon_2|+(1-a)|\epsilon_3|\\
			& \leq |\epsilon|+|\epsilon\sqrt{a}|+|\epsilon\sqrt{1-a}|\leq 3\epsilon
		\end{flalign*}
		If $\rho$ is the difference in gain in noise free case, under noise, difference in gain becomes, $\rho(1-2\eta)$. Thus we can set $\epsilon=\rho (1-2\eta)/6$ in eq.~\ref{eq:prob_union} for both of the splits to get the probability bound.\\
		
		$\bullet${Twoing Rule:}
		Similarly for twoing rule we bound the gain assuming \lq all fractions are $\epsilon$-accurate\rq$\;$event.
		We get, after simplification,
		\begin{flalign*}
			|\hat{\mbox{G}}_{\mbox{Twoing}}^{\eta}(f)-\mbox{G}_{\mbox{Twoing}}^{\eta}(f)|&
			\leq a(1-a)(|\epsilon_2-\epsilon_3|)(|p_l^{\eta}-p_r^{\eta}|)\\
			&\leq (1-2\eta)(|\epsilon (1-a)\sqrt{a}|+|\epsilon a\sqrt{1-a}|)(|p_l-p_r|)\\
			&\leq\frac{(1-2\eta)}{2}(|\epsilon|+|\epsilon|)\leq (1-2\eta)\epsilon
		\end{flalign*}
		Note $\sqrt{a}\sqrt{1-a}\leq 1/2$. Under noise, difference of gain becomes $(1-2\eta)^2\rho$. Here we can set $\epsilon=\rho(1-2\eta)/2$ to bound the probability of ordering change.\\
		
		Thus for all cases, required sample size in parent node is $n\geq \mathcal{O}(\frac{1}{\rho ^2(1-2\eta)^2}\ln (\frac{1}{\delta}))$			
			\label{apd:second}
		\end{proof}

\end{document}